\documentclass[letterpaper, 11 pt]{article} 
\usepackage{amsmath,amsfonts,amssymb,color}
\usepackage{mathtools}
\usepackage{dsfont}
\usepackage[dvipsnames]{xcolor}
\definecolor{mygreen}{rgb}{0.0, 0.5, 0.0}
\definecolor{winered}{rgb}{0.8,0,0}
\definecolor{myblue}{rgb}{0,0,0.8}
\usepackage{tikz}
\usepackage{amsthm}
\usepackage{empheq}

\newtheorem{definition}{Definition}
\newtheorem{theorem}{Theorem}
\newtheorem{lemma}{Lemma}

\newtheorem{remark}{Remark}
\newtheorem{assumption}{Assumption}

\newcommand{\mc}{\mathcal}

\DeclarePairedDelimiterX{\norm}[1]{\lVert}{\rVert}{#1}

\usepackage{algorithm}
\usepackage{algpseudocode}
\usepackage{epsfig}
\usepackage{array}
\usepackage{multirow}
\usepackage{epstopdf}
\usepackage{tikz}
\usepackage{relsize}
\usetikzlibrary{shapes,arrows}
\usepackage[hidelinks]{hyperref}
\hypersetup{
    colorlinks=true,
    linkcolor={winered},
    citecolor={mygreen}
}
\usepackage{cite}
\usepackage[margin=1in]{geometry}
\title{A Simple Finite-Time Analysis of TD Learning with Linear Function Approximation}
\author{Aritra Mitra
\thanks{A. Mitra is with the Department of Electrical and Computer Engineering,  North Carolina State University. Email: {\tt amitra2@ncsu.edu}.}}
\date{}
\begin{document}
\maketitle
\thispagestyle{empty}
\pagestyle{empty}
\begin{abstract}
We study the finite-time convergence of TD learning with linear function approximation under Markovian sampling. Existing proofs for this setting either assume a projection step in the algorithm to simplify the analysis, or require a fairly intricate argument to ensure stability of the iterates. We ask: \textit{Is it possible to retain the simplicity of a projection-based analysis without actually performing a projection step in the algorithm?} Our main contribution is to show this is possible via a novel two-step argument. In the first step, we use induction to prove that under a standard choice of a constant step-size $\alpha$, the iterates generated by TD learning remain uniformly bounded in expectation. In the second step, we establish a recursion that mimics the steady-state dynamics of TD learning up to a bounded perturbation on the order of $O(\alpha^2)$ that captures the effect of Markovian sampling. Combining these pieces leads to an overall approach that considerably simplifies existing proofs. We conjecture that our inductive proof technique will find applications in the analyses of more complex stochastic approximation algorithms, and conclude by providing some examples of such applications.  
\end{abstract}
\section{Introduction}
We study sequential decision-making within the framework of a Markov Decision Process (MDP) with a finite state and action space. At each time-step, an agent/learner interacts with an environment by playing an action, observing a reward for this action, and transitioning to a new state in the MDP. The reward functions and probability transition kernels of the MDP that generate the agent's observations are \emph{unknown} to the agent. Via repeated interactions with the environment, the goal of the agent is to learn a policy (sequence of actions) that maximizes a long-term cumulative return. In this paper, we focus on the simpler problem of \emph{policy evaluation}: estimating the expected infinite-horizon cumulative discounted return - known as the value function - corresponding to a \emph{fixed} policy. To solve this policy evaluation problem, Sutton introduced a family of on-line incremental algorithms known as \emph{temporal-difference} (TD) methods in his 1988 paper~\cite{sutton1988learning}. While these algorithms are easy to describe and implement, their analyses turn out to be quite non-trivial: in~\cite{tsitsiklisroy}, Tsitsiklis and Van Roy note: ``\emph{Though temporal-difference learning is simple and elegant, a rigorous analysis of its behavior requires significant sophistication}". This is more so the case when in practice, due to large state spaces, a function approximator is used to approximate the value function. In this context, our goal is to provide a short and accessible \emph{finite-time} convergence proof of TD learning with linear function approximation. Before we explain our contribution in this regard, it is instructive to briefly summarize what is known for this setting. 
\newpage
\textbf{Related Work.} The first paper to provide an asymptotic convergence analysis for TD learning with linear function approximation was~\cite{tsitsiklisroy}. This was achieved by viewing TD methods as instances of stochastic approximation algorithms~\cite{borkar, borkarode}. While the results in~\cite{tsitsiklisroy} provided foundational insights, they came with no rates. In the years to come, several papers did manage to provide finite-time convergence rates for TD learning~\cite{korda, narayanan, lakshmi, dalal}; however, their analyses made the restrictive assumption that the data samples used for performing updates are generated in an i.i.d. manner over time from the stationary distribution of the underlying Markov chain (induced by the policy). In reality, however, these samples are all part of a single trajectory generated by the Markov chain, and, as such, exhibit \emph{temporal correlations}. It is precisely these temporal correlations that make the analysis of even the simplest TD method - known as \texttt{TD}(0) - quite non-trivial.

By making interesting connections to the dynamics of stochastic gradient descent (SGD), the authors in~\cite{bhandari_finite} were able to provide the first finite-time analysis of TD learning under Markovian sampling. However, their analysis hinges crucially on a projection step in the algorithm to ensure that the iterates generated by the projected TD method remain uniformly bounded. To summarize, all the papers above either only provide asymptotic rates, or assume i.i.d. sampling, or assume a projection step. To our knowledge, the first paper to provide finite-time mean-square error bounds for TD learning with linear function approximation under Markovian sampling \emph{without 
a projection step} was~\cite{srikant2019finite}. The approach in~\cite{srikant2019finite} is control-theoretic, where the authors draw on Lyapunov theory for analyzing the stability of linear dynamical systems. While the analysis in~\cite{srikant2019finite} is elegant, it requires a relatively more involved argument than the simpler projection-based analysis in~\cite{bhandari_finite}. This leads to the main question we investigate in this paper: \emph{Is it possible to retain the simplicity of a projection-based analysis without actually performing a projection step in the algorithm?} 

\textbf{Our Contribution.} We start by answering the above question in the affirmative for the \texttt{TD}(0) algorithm with linear function approximation. Our proof is simple, and relies on a novel inductive argument. In what follows, we provide the crux of the argument; the details are deferred to Section~\ref{sec:analysis}. Like in any standard stochastic optimization proof, we first use the update rule to write down a recursion for the mean-squared error. With little algebra, the right hand side of this recursion can be effectively decomposed into three terms: (i) an exponentially decaying term that captures the ``steady-state" dynamics of \texttt{TD}(0); (ii) a noise variance term that is typical of any noisy iterative algorithm; and (iii) a term that captures an additional error due to Markovian sampling. 

In the absence of the third term, the proof would be standard, akin to an SGD (stochastic gradient descent) analysis under i.i.d. noise. We view the third term as a disturbance; the rationale for this will become clear soon. Unlike an input-to-stability argument in control where a uniformly bounded disturbance excites a stable system, the disturbance in our case depends on the \emph{time-varying iterates}. At this stage, a projection step would greatly simplify the analysis since we could then argue boundedness of the iterates, and hence, of the disturbance. However, recall that we do not assume a projection step. So is there a simple way then to control the disturbance term? Yes, this is where our novel inductive idea kicks in. By assuming a suitable uniform bound in expectation - say $B$ - on the past iterates as part of our induction hypothesis, we bound the disturbance term. Plugging this bound back into our main recursion, we are able to then show that the same uniform bound $B$ applies to the new iterate. In short, we show in Theorem~\ref{thm:ind_main} that under a standard choice of a constant step-size $\alpha$, the iterates generated by \texttt{TD}(0) remain uniformly bounded in expectation. The proof of this result is the main contribution of our work. Armed with Theorem~\ref{thm:ind_main}, we go back to the main recursion with the knowledge that the disturbance term is an $O(\alpha^2)$ uniformly bounded perturbation, of the same order as the noise variance. The rest is trivial. 

\textbf{Motivation and Applications.} One might ask: \textit{Why care about this new proof technique?} Here are our reasons. First, each new analysis of TD learning - like ours - sheds new insights into the dynamics of a rather complex stochastic process. Second, as we discuss in Section~\ref{sec:applications}, the scope of our inductive proof technique extends well beyond \texttt{TD}(0) to a much broader class of general (potentially nonlinear) stochastic approximation schemes that include the \texttt{TD}($\lambda$) family (with linear function approximation) and variants of Q-learning as special cases. Finally, perhaps the most compelling reason is the following. Iterative optimization algorithms like SGD form the cornerstone of almost all large-scale machine learning applications for a reason: SGD is \emph{provably robust} to a variety of structured perturbations that invariably show up in these applications, e.g., delays, asynchrony, quantization and compression errors, and adversarial corruption~\cite{arjevani, stich2020comm}. However, it remains poorly understood whether even basic reinforcement learning (RL) algorithms like the TD methods are robust to similar perturbations. In principle, one can view the effect of these perturbations as a disturbance (potentially iterate-dependent) to the nominal dynamics, much like the effect of Markovian noise. Lumping all the disturbances together, a natural next step could then be to apply the inductive proof technique we outlined earlier. In a companion paper~\cite{delayedSA}, we show that this approach is essential to deriving tight rates for a broad class of stochastic approximation algorithms perturbed by time-varying delays and Markovian noise - a challenging setting that was previously unexplored. In Section~\ref{sec:applications}, we also explain that it is unclear whether the existing approach in~\cite{srikant2019finite} can handle such time-varying delays; in fact, this was precisely the motivation for developing the technique in this paper. Thus, we believe that the simplicity of our approach can help reason about the robustness of various complex stochastic approximation algorithms, beyond what we cover in this note. 

\section{Background on TD Learning}
\label{sec:model}
In this section, we set up notation and provide the necessary technical background. Given a positive integer $n$, we use $[n]$ to denote the set $\{1, 2, \ldots, n\}.$ Unless otherwise stated, we will use $\Vert \cdot \Vert$ to denote the standard Euclidean norm. We consider a Markov Decision Process (MDP) 
denoted by $\mc{M}=(\mc{S},\mc{A},\mc{P},\mc{R},\gamma)$,  where $\mc{S}$ is a finite state space of size $n$, $\mc{A}$ is a finite action space, $\mc{P}$ is a set of action-dependent Markov transition kernels, $\mc{R}$ is a reward function, and $\gamma \in (0,1)$ is the discount factor. A deterministic policy $\mu: \mc{S} \rightarrow \mc{A}$ is a mapping from the states to the actions. When a fixed policy $\mu$ interacts with the underlying MDP, it generates a Markov reward process (MRP) characterized by a transition matrix $P_{\mu}$, and a reward function $R_{\mu}$. At a given state $s$, upon playing the action $\mu(s)$, an agent receives an expected instantaneous reward denoted by $R_{\mu}(s)$, and its probability of transitioning from state $s$ to state $s'$ is given by $P_{\mu}(s,s')$. 
The discounted expected cumulative reward obtained by playing policy $\mu$ starting from initial state $s$ is given by:
\begin{equation}
    V_{\mu}(s) = \mathbb{E}\left[ \sum_{t=0}^{\infty}\gamma^t R_{\mu}(s_t) | s_0 = s \right],
\label{eqn:v_cum}
\end{equation}
where $s_t$ represents the state of the Markov chain (induced by $\mu$) at time $t$, when initiated from $s_0=s$. In essence, the value function $V_\mu$ measures the ``goodness" of the policy $\mu$, and the central goal in RL is to find an optimal policy that simultaneously maximizes $V_\mu(s), \forall s \in \mathcal{S}.$ For an MDP with finite state and action spaces, such a (deterministic) optimal policy is known to always exist~\cite{puterman}. We will primarily focus on the simpler task of \textit{policy evaluation}, where the goal is to evaluate the value function $V_{\mu}$ corresponding to a fixed policy $\mu$.\footnote{Later in Section~\ref{sec:applications}, we will comment on how our developments also aid the problem of finding the optimal policy.} It is well-known \cite{tsitsiklisroy} that $V_\mu$ is the fixed point of the policy-specific Bellman operator $\mathcal{T}_\mu:\mathbb{R}^{n} \rightarrow \mathbb{R}^{n}$, i.e., $\mathcal{T}_\mu V_\mu = V_\mu$, where for any $V \in \mathbb{R}^n$,
\begin{equation}
    (\mc{T}_\mu V) (s) = R_{\mu}(s)+\gamma \sum_{s'\in \mc{S}} P_{\mu}(s,s') V(s'), \hspace{1mm} \forall s \in \mc{S}.
\label{eqn:Bellman_op}
\end{equation}

When the underlying MDP is known, the above key property is sufficient to devise a simple dynamic programming approach that guarantees convergence to the value function $V_{\mu}$~\cite{puterman}. Our interest is however in the RL setting where the state transition matrices and the reward functions of the MDP are \emph{unknown}. In addition to this challenge, for contemporary RL applications, the size of the state space $\mc{S}$ can be extremely large. This renders the task of estimating $V_\mu$ \textit{exactly} (based on observations of rewards and state transitions) intractable. The common workaround is to consider a parametric approximation $\hat{V}_\theta$ of $V_\mu$ in the linear subspace spanned by a set $\{\phi_k\}_{k\in [K]}$ of $K \ll n$  basis vectors, where $\phi_k =[\phi_k(1), \ldots, \phi_k(n)]^{\top} \in \mathbb{R}^{n}$. Specifically, we have $\hat{V}_\theta(s) = \sum_{k=1}^{K} \theta(k)\phi_k(s),$ 
where $\theta = [\theta(1), \ldots, \theta(K)]^{\top} \in \mathbb{R}^{K}$ is a weight/parameter vector. Let $\Phi \in \mathbb{R}^{n \times K}$ be a matrix with $\phi_k$ as its $k$-th column; we then have $\hat{V}_\theta=\Phi \theta$. Let us also denote the $s$-th row of $\Phi$ by $\phi(s) \in \mathbb{R}^{K}$, and refer to it as the feature vector for state $s$. For each state $s \in \mathcal{S}$, we then have: $\hat{V}_\theta(s) = \langle \phi(s), \theta \rangle$. To proceed, we will make the standard assumption that the columns of $\Phi$ are linearly independent, and that the feature vectors are normalized, i.e., for each $s \in \mc{S}$, $\Vert \phi(s) \Vert^2 \leq 1$~\cite{bhandari_finite}. Given this premise, the problem of interest is to find the parameter $\theta^*$ corresponding to the best parametric approximation (in a suitable norm) of $V_{\mu}$. We now describe the classical TD(0)  algorithm~\cite{sutton1988learning} - due to Sutton - for achieving this goal. 

\textbf{The \texttt{TD}(0) Algorithm.} Starting from an initial parameter estimate $\theta_0$, the TD(0) algorithm operates as follows.  At each time-step $t=0, 1, \ldots$, an observation in the form of a data tuple $X_t=(s_t, s_{t+1}, r_t=R_{\mu}(s_t))$  is received. The tuple comprises the current state $s_t$, the next state $s_{t+1}$ reached by playing action $\mu(s_t)$, and the instantaneous reward $r_t$. Given this tuple $X_t$, the current parameter $\theta_t$ is updated by moving along the TD(0) update direction; for a fixed $\theta \in \mathbb{R}^K$, we define this direction $g_t(\theta)=g(\theta; X_t)$ as follows:
$$ g_t(\theta) \triangleq \left(r_t + \gamma \langle \phi(s_{t+1}), \theta \rangle -  \langle \phi(s_{t}), \theta \rangle\right) \phi(s_t), \forall \theta \in \mathbb{R}^K. $$ 
The TD(0) update rule can then be described as
\begin{equation}
    \theta_{t+1}=\theta_t + \alpha_t g_t(\theta_t),
\label{eqn:TD(0)update}
\end{equation}
where $\alpha_t \in (0,1)$ is the step-size/learning rate. \textit{Our main {goal} in this paper is to provide a short non-asymptotic proof of convergence for the above algorithm.} To do so, we will make the following standard assumption~\cite{tsitsiklisroy, bhandari_finite, srikant2019finite}.

\begin{assumption}
\label{ass:aperiodic}
    The Markov chain induced by the policy $\mu$ is aperiodic and irreducible.
\end{assumption}

For the value functions to be well-defined, we will also make the standard assumption that $\exists \bar{r} > 0$ such that $R_{\mu}(s) \leq \bar{r}, \forall s \in \mathcal{S}.$ Under these assumptions, Tsitsiklis and Van Roy showed that with a suitable step-size sequence $\{\alpha_t\}$, the iterates generated by Eq.~\eqref{eqn:TD(0)update} converge almost surely to the best linear approximator of $V_{\mu}$ in the span of $\{\phi_k\}_{k\in [K]}$~\cite{tsitsiklisroy}. To be more precise, we note that under Assumption~\ref{ass:aperiodic}, the Markov chain induced by the policy $\mu$ admits a unique stationary distribution $\pi$~\cite{levin2017markov}. Let $D$ be a diagonal matrix with $D(i,i)=\pi(i), \forall i \in [n].$ Moreover, let $\Pi_D(\cdot)$ denote the projection operator onto the subspace spanned by $\{\phi_k\}_{k\in [K]}$ with respect to the inner product $\langle \cdot, \cdot \rangle_D$. Then, the main result in~\cite{tsitsiklisroy} shows that $\theta_t \rightarrow \theta^*$ with probability $1$, where $\theta^*$ is the unique solution of the projected Bellman equation $\Pi_D \mc{T}_\mu (\Phi \theta^*) = \Phi \theta^*$. Notably, this result is \emph{asymptotic}: it does not provide a sense of the \emph{rate} at which $\theta_t$ approaches $\theta^*$ as a function of the discrete time-index $t$. An object that provides a lot of intuition about the rate is the ``steady-state" version of the \texttt{TD}(0) update direction, defined as follows:
\begin{equation}
\bar{g}(\theta) \triangleq \mathbb{E}_{s_t \sim  \pi,  s_{t+1} \sim P_{\mu}(\cdot| s_t) }\left[g(\theta; X_t)\right], \forall \theta\in\mathbb{R}^K.
\label{eqn:steady-stateTD}
\end{equation}

In~\cite{bhandari_finite}, it was shown that the \emph{deterministic} steady-state recursion $\theta_{t+1}= \theta_{t}+\alpha \bar{g}(\theta_t)$ converges \emph{linearly} to $\theta^*$ with a suitable constant step-size $\alpha$. To extend this result to the stochastic recursion in~\eqref{eqn:TD(0)update}, we will require the notion of a mixing time.

\begin{definition} \label{def:mix} 
Define
$\tau_{\epsilon} \triangleq \min\{t\geq1: \Vert \mathbb{E}\left[g(\theta; X_k)|X_0\right]-\bar{g}(\theta)\Vert \leq \epsilon\left(\Vert \theta \Vert +1 \right), \forall k \geq t, \forall \theta \in \mathbb{R}^K, \forall X_0\}.$ 
\end{definition}

A key implication of Assumption \ref{ass:aperiodic} is that the total variation distance between the conditional distribution $\mathbb{P}\left(s_t=\cdot|s_0=s\right)$ and the stationary distribution $\pi$ decays geometrically fast, regardless of the initial state $s\in\mc{S}$~\cite{levin2017markov}. This, in turn, immediately implies that $\tau_{\epsilon}$ in Definition \ref{def:mix} is $O\left(\log(1/\epsilon)\right)$ \cite{chenQ}. For our purpose, we will set the precision $\epsilon=\alpha$, and henceforth simply use $\tau$ as a shorthand for $\tau_{\alpha}$. By exploiting the geometric mixing property above, and by making elegant connections to smooth and strongly convex optimization using (stochastic) gradient descent, the authors in~\cite{bhandari_finite} were able to provide a finite-time convergence rate for \texttt{TD}(0). However, as explained in the Introduction, the analysis in~\cite{bhandari_finite} crucially relies on a projection step to control the effect of temporal correlations in the data tuples - a consequence of Markovian sampling. \emph{Can we continue to leverage the insights from optimization in~\cite{bhandari_finite}, while analyzing the \texttt{TD}(0) update rule in~\eqref{eqn:TD(0)update} without projection?} The next section answers this question in the affirmative. 

\section{Convergence Analysis}
\label{sec:analysis}
We start by compiling a few basic results that will aid our subsequent analysis. 
The first such result provides the connection to optimization: it shows that the steady-state \texttt{TD}(0) update direction $\bar{g}(\theta)$ acts like a ``pseudo-gradient", driving the \texttt{TD}(0) iterates towards the solution $\theta^*$ of the projected Bellman equation. A proof of this result was provided in~\cite{bhandari_finite}. We provide an alternate proof in the Appendix to keep the paper self-contained.

\begin{lemma}
\label{lemma:convex}
The following holds $\forall \theta \in \mathbb{R}^K$:
$$ \langle \theta^* - \theta, \bar{g}(\theta) \rangle \geq \omega (1-\gamma) \Vert \theta^* -\theta \Vert^2,$$ 
 where $\omega$ is the smallest eigenvalue of the matrix $\Sigma = \Phi^\top D \Phi$.
\end{lemma}
\noindent Under the assumptions on the feature matrix $\Phi$ in Section~\ref{sec:model}, and Assumption~\ref{ass:aperiodic}, it is easy to see that $\Sigma$ is positive definite with $\omega \in (0,1).$ We will make use of the fact that the \texttt{TD}(0) update direction, along with its steady-state version, are both $2$-Lipschitz, i.e., $\forall t \in \mathbb{N}$, and $\forall \theta_1, \theta_2 \in \mathbb{R}^K$, 
\begin{equation}
\max\{\Vert \bar{g}(\theta_1)-\bar{g}(\theta_2) \Vert, \Vert {g}_t(\theta_1)- {g}_t(\theta_2) \Vert\} \leq 2 \Vert \theta_1 - \theta_2 \Vert. 
\label{eqn:Lipschitz}
\end{equation}

Next, at several points in our analysis, we will invoke the following bound on the norm of the \texttt{TD}(0) update direction: 
\begin{equation}
    \Vert g_t(\theta) \Vert \leq 2 \Vert \theta \Vert + 2 \bar{r} \leq 2\Vert \theta \Vert + 2 \sigma, \forall t\in \mathbb{N}, \forall \theta \in \mathbb{R}^K, 
    \label{eqn:nrm}
\end{equation}
where $\sigma=\max\{1,\bar{r},\Vert \theta^* \Vert\}.$ This immediately yields:
\begin{equation}
 \Vert g_t(\theta) \Vert \leq 2 \Vert \theta - \theta^* \Vert + 4 \sigma, \forall t\in \mathbb{N}, \forall \theta \in \mathbb{R}^K.
 \label{eqn:nrmbnd}
\end{equation}

The bounds in Eq.~\eqref{eqn:Lipschitz} and Eq.~\eqref{eqn:nrm} follow straightforwardly from the fact that the TD(0) update direction $g_t(\theta)$ is affine in the parameter $\theta$. Hence, we omit the proof here; an interested reader can take a look at~\cite{mitraTDEF}. Since $\bar{g}(\theta^*)=0$ (see~\cite{tsitsiklisroy}), the Lipschitz property in \eqref{eqn:Lipschitz} implies 
\begin{equation}
\Vert \bar{g}(\theta) \Vert \leq 2 \Vert \theta - \theta^*\Vert, \forall \theta \in \mathbb{R}^K.
\label{eqn:nrmbnd2}
\end{equation}

Let us now provide some intuition behind our analysis. Throughout, to present our arguments in a clean way, we will use the big-$O$ notation to suppress universal constants. We begin by defining a couple of objects for all $t\geq 0$: 
\begin{equation}
d_t \triangleq \mathbb{E}[ \Vert \theta_t - \theta^*\Vert^2 ], \hspace{2mm} e_t \triangleq \mathbb{E}[\langle \theta_t -\theta^*, g_t(\theta_t) - \bar{g}(\theta_t)\rangle]. 
\label{eqn:disturbance}
\end{equation}
Now observe from the update rule~\eqref{eqn:TD(0)update} that
\begin{equation}
\begin{aligned}
    \Vert \theta_{t+1} - \theta^*\Vert^2 &= \Vert \theta_t - \theta^* \Vert^2 +2 \alpha \langle \theta_t - \theta^*, g_t(\theta_t) \rangle + \alpha^2 \Vert g_t(\theta_t) \Vert^2\\
    &= \Vert \theta_t - \theta^* \Vert^2 +2 \alpha \langle \theta_t - \theta^*, \bar{g}(\theta_t) \rangle + \alpha^2 \Vert g_t(\theta_t) \Vert^2\\
    &\hspace{2mm} + 2\alpha \langle \theta_t -\theta^*, g_t(\theta_t) - \bar{g}(\theta_t)\rangle\\
    & \leq \left(1-2 \alpha \omega (1-\gamma) + 8 \alpha^2 \right) \Vert \theta_t - \theta^* \Vert^2 + 32 \alpha^2 \sigma^2\\
    &\hspace{2mm} +2\alpha \langle \theta_t -\theta^*, g_t(\theta_t) - \bar{g}(\theta_t)\rangle,
\end{aligned}
\nonumber
\end{equation}
where in the last step, we used Lemma~\ref{lemma:convex} and Eq.~\eqref{eqn:nrmbnd}. Taking expectations on both sides of the above display then yields:
\begin{equation}
\boxed{
d_{t+1} \leq \underbrace{\left(1-2 \alpha \omega (1-\gamma) + 8 \alpha^2 \right) d_t}_{T_1} + \underbrace{32 \alpha^2 \sigma^2}_{T_2}+\underbrace{2\alpha e_t}_{T_3}.
\label{eqn:main_recursion}}
\end{equation}
We are left to analyze the key recursion above. Notice that the right hand side of the above recursion features three terms: (i) the term $T_1$ captures the steady-state behavior of \texttt{TD}(0); (ii) the term $T_2$ is a noise variance term that typically shows up in the analysis of any noisy iterative algorithm (e.g., SGD); and (iii) the term $T_3$ captures the effect of Markovian sampling. In the absence of the third term, one can immediately see from Eq.~\eqref{eqn:main_recursion} that with a suitably chosen step-size $\alpha$, the iterates would converge linearly (in the mean-square sense) to a ball of radius $O(\alpha \sigma^2)$ centered around the optimal parameter $\theta^*$. Moreover, the proof would be near-identical to that of SGD. Intuitively, if we could thus show that $T_3 = O(\alpha^2 \sigma^2)$, i.e., the same order as $T_2$, we would be done. However, observe from Eq.~\eqref{eqn:disturbance} that $e_t$ depends on the current iterate $\theta_t$; as such, we cannot directly claim a uniform upper bound on $T_3$ in the absence of a projection step. For a moment, suppose we wish to overcome this difficulty via an inductive argument where we assume a uniform upper bound - say $B$ - on all the iterates up to time $t$. Our goal would then be to appeal to Eq.~\eqref{eqn:main_recursion} to show that the same bound $B$ applies to the iterate at time $t+1$. Unfortunately, a naive inductive argument such as the one above would only tell us that $T_3$ is an $O(\alpha)$ perturbation. The issue with this bound is that it is too loose: we wanted an $O(\alpha^2)$ additive perturbation, but ended up with an $O(\alpha)$ perturbation. The reason why we fell short of our desired outcome is because we did not exploit the additional structure in $e_t$: the geometric mixing property in Definition~\ref{def:mix}
tells us that eventually, $e_t$ should be ``small". Thus, we need a finer inductive argument that leverages this fact. 

Our first step in providing such an argument is a result that simply states that for a step-size $\alpha$ that scales inversely with the mixing time $\tau$, the iterate sequence $\{\theta_t\}$ will remain uniformly bounded for the first $\tau$ time-steps. This result will serve as the base case of our subsequent induction argument. 

\begin{lemma} \label{lemma:base}
Suppose $\alpha \leq 1/(8\tau).$ Define $B \triangleq 10 \max\{\Vert \theta_0 - \theta^* \Vert^2, \sigma^2\}.$ Then, we have:
\begin{equation}
\Vert \theta_k - \theta^* \Vert^2 \leq B, \forall k \in [\tau]. 
\end{equation}
\end{lemma}
\begin{proof}
From the update rule~\eqref{eqn:TD(0)update}, we have
\begin{equation}
\begin{aligned}
\Vert \theta_{t+1} - \theta^*\Vert &\leq \Vert \theta_t - \theta^* \Vert + \alpha \Vert g_t(\theta_t) \Vert \\
&\leq (1+2\alpha) \Vert \theta_t - \theta^* \Vert + 4 \alpha \sigma,
\end{aligned}
\end{equation}
where we used~\eqref{eqn:nrmbnd} in the second step. Iterating the above inequality yields the following $\forall k \in [\tau]$:
\begin{equation}
\begin{aligned}
\Vert \theta_{k} - \theta^*\Vert & \leq (1+2\alpha)^k \Vert \theta_0 - \theta^* \Vert + 4 \alpha \sigma \sum_{j=0}^{k-1} (1+2\alpha)^j\\
&\leq (1+2\alpha)^\tau \Vert \theta_0 - \theta^* \Vert + 4 \alpha \tau (1+2\alpha)^\tau\sigma\\
&\leq 2\Vert \theta_0 - \theta^* \Vert + 8 \alpha \tau \sigma\\
&\leq 2\Vert \theta_0 - \theta^* \Vert + \sigma,
\end{aligned}
\nonumber
\end{equation}
where in the third step, we used $(1+x) \leq \exp(x), \forall x \in \mathbb{R}$ to deduce that $(1+2\alpha)^\tau \leq \exp(0.25) < 2$, for $\alpha \leq 1/(8\tau).$ Squaring both sides of the final inequality above leads to the desired claim. 
\end{proof}

Our next goal is to show that a bound akin to that in the above lemma applies to time-steps greater than $\tau$ as well. To that end, we need the following intermediate result.

\begin{lemma} \label{lemma:driftbnd}
Consider any $t \geq \tau.$ Suppose $d_k \leq B, \forall k \in [t].$ Then, the following is true:
$$ \mathbb{E}[ \Vert \theta_t - \theta_{t-\tau} \Vert^2 ] \leq O(\alpha^2 \tau^2 B).$$
\end{lemma}
\begin{proof} Observe:
\begin{equation}
\begin{aligned}
\Vert \theta_t - \theta_{t-\tau} \Vert &\leq \sum_{k=t-\tau}^{t-1} \Vert \theta_{k+1}-\theta_k \Vert \\
& \leq \alpha \sum_{k=t-\tau}^{t-1} \Vert g_k(\theta_k) \Vert\\
& \leq O(\alpha) \sum_{k=t-\tau}^{t-1} \left( \Vert \theta_k-\theta^* \Vert + \sigma \right),
\end{aligned}
\end{equation}
where in the last step, we used~\eqref{eqn:nrmbnd}. Squaring both sides of the above inequality, and taking expectations, we obtain
\begin{equation}
\begin{aligned}
\mathbb{E}[ \Vert \theta_t - \theta_{t-\tau} \Vert^2 ] &\leq O(\alpha^2 \tau) \sum_{k=t-\tau}^{t-1} (d_k + \sigma^2)\\
&\leq O(\alpha^2 \tau^2 B),
\end{aligned}
\end{equation}
where we used $d_k \leq B, \, \forall k \in [t]$, and $\sigma^2 \leq B$. 
\end{proof}

The final piece we need to complete our main inductive argument is the following lemma. 

\begin{lemma} \label{lemma:mixing}
Consider any $t \geq \tau.$ Suppose $d_k \leq B, \forall k \in [t].$ Then, the following is true:
$ e_t \leq O(\alpha \tau B).$
\end{lemma}
\begin{proof}
Let us start with the following decomposition: $\langle \theta_t -\theta^*, g_t(\theta_t) - \bar{g}(\theta_t)\rangle=T_1+T_2+T_3+T_4$, where
\begin{equation}
\begin{aligned}
T_1&=\langle \theta_t-\theta_{t-\tau}, g_t(\theta_t)-\bar{g}(\theta_t)\rangle,\\
T_2&=\langle \theta_{t-\tau}-\theta^*, g_t(\theta_{t-\tau})-\bar{g}(\theta_{t-\tau})\rangle,\\
T_3&=\langle \theta_{t-\tau}-\theta^*, g_t(\theta_t)- g_t(\theta_{t-\tau})\rangle, \hspace{2mm} \textrm{and}\\
T_4&=\langle \theta_{t-\tau}-\theta^*,\bar{g}(\theta_{t-\tau})-\bar{g}(\theta_t)\rangle.\\
\end{aligned}
\nonumber
\end{equation}
We will now argue that $\mathbb{E}[T_i] \leq O(\alpha \tau B), \forall i \in \{1,2,3,4\}.$ For $T_1$, observe that:
\begin{equation}
\begin{aligned}
T_1 &\leq \Vert \theta_t - \theta_{t-\tau} \Vert \Vert g_t(\theta_t) - \bar{g}(\theta_t) \Vert \\
& \leq \frac{1}{2\alpha \tau} \Vert \theta_t - \theta_{t-\tau} \Vert^2 + \frac{\alpha \tau}{2} \Vert g_t(\theta_t) - \bar{g}(\theta_t) \Vert^2 \\
&\leq \frac{1}{2\alpha \tau} \Vert \theta_t - \theta_{t-\tau} \Vert^2 + \alpha \tau \left( \Vert g_t(\theta_t) \Vert^2 + \Vert \bar{g}(\theta_t) \Vert^2\right)\\
&\overset{\eqref{eqn:nrmbnd}, \eqref{eqn:nrmbnd2}}\leq \frac{1}{2\alpha \tau} \Vert \theta_t - \theta_{t-\tau} \Vert^2 + O(\alpha \tau) \left( \Vert \theta_t-\theta^* \Vert^2 + \sigma^2 \right). 
\end{aligned}
\nonumber
\end{equation}
Now taking expectations on both sides of the above inequality, invoking Lemma~\ref{lemma:driftbnd}, and using $d_t \leq B$, we conclude that $\mathbb{E}[T_1] \leq O(\alpha \tau B).$ 

Next, for $T_3$, we have 
\begin{equation}
\begin{aligned}
T_3 &\leq \Vert \theta_{t-\tau} - \theta^*\Vert \Vert g_t(\theta_t) - {g}_t(\theta_{t-\tau}) \Vert\\
&\overset{\eqref{eqn:Lipschitz}}\leq 2 \Vert \theta_{t-\tau} - \theta^*\Vert \Vert \theta_t - \theta_{t-\tau} \Vert\\
&\leq \frac{1}{\alpha \tau} \Vert \theta_t -\theta_{t-\tau} \Vert^2 + \alpha \tau \Vert \theta_{t-\tau}-\theta^* \Vert^2. 
\end{aligned}
\nonumber
\end{equation}
Taking expectations on both sides of the above inequality and using Lemma~\ref{lemma:driftbnd} yields:
$$
\mathbb{E}[T_3] \leq O(\alpha \tau B) + \alpha \tau d_{t-\tau} \leq O(\alpha \tau B),
$$
where in the last step, we used $d_k \leq B, \forall k \in [t].$ The fact that $\mathbb{E}[T_4] \leq O(\alpha \tau B)$ follows exactly the same analysis as above. We are left to bound $T_2$; \emph{this is the only place in the entire proof where we will exploit the geometric mixing property of the underlying Markov chain.} We proceed as follows.
\begin{equation}
\begin{aligned}
    \mathbb{E}\left[T_{2}\right] &= \mathbb{E}\left[\langle \theta_{t-\tau} -\theta^*, g_t(\theta_{t-\tau})-\bar{g}(\theta_{t-\tau})\rangle\right]\\
    &=\mathbb{E}\left[\mathbb{E}\left[\langle \theta_{t-\tau} -\theta^*, g_t(\theta_{t-\tau})-\bar{g}(\theta_{t-\tau})\rangle | \theta_{t-\tau}, X_{t-\tau}\right]\right]\\
    &=\mathbb{E}\left[\langle \theta_{t-\tau} -\theta^*, \mathbb{E}\left[ g_t(\theta_{t-\tau})-\bar{g}(\theta_{t-\tau})| \theta_{t-\tau}, X_{t-\tau}\right]\rangle\right]\\
    &\leq \mathbb{E}\left[\Vert \theta_{t-\tau} -\theta^*\Vert \Vert \mathbb{E}\left[ g_t(\theta_{t-\tau})-\bar{g}(\theta_{t-\tau})| \theta_{t-\tau}, X_{t-\tau}\right]\Vert\right]\\
    &\overset{(a)}\leq \alpha \mathbb{E}\left[\Vert \theta_{t-\tau} -\theta^*\Vert  \left(1+\Vert \theta_{t-\tau}\Vert\right)\right]\\
    &\leq \alpha \mathbb{E}\left[\Vert \theta_{t-\tau} -\theta^*\Vert \left(1+\Vert \theta^* \Vert + \Vert \theta_{t-\tau}-\theta^*\Vert\right)\right]\\
    &\overset{(b)}\leq \alpha \mathbb{E}\left[\Vert \theta_{t-\tau} -\theta^*\Vert \left(2\sigma+\Vert \theta_{t-\tau} -\theta^* \Vert\right)\right]\\
    & \leq O(\alpha) \mathbb{E}\left[\Vert \theta_{t-\tau} -\theta^*\Vert^2 + \sigma^2\right]\\
    &\overset{(c)} \leq O(\alpha) \left(d_{t-\tau}+\sigma^2\right)=O(\alpha B),\\
\end{aligned}
\nonumber
\end{equation}
where (a) follows from the definition of the mixing time $\tau$ in Definition~\ref{def:mix}; (b) follows from recalling that $\sigma=\max\{1,\bar{r}, \Vert\theta^* \Vert\}$; and (c) follows from us again using  $d_k \leq B, \forall k \in [t]$. 
\end{proof}

We are now ready to state and prove the key technical result that guarantees uniform boundedness of the iterates under a suitable choice of the step-size $\alpha$. 

\begin{theorem} (\textbf{Boundedness of Iterates}) \label{thm:ind_main} There exists a universal constant $C \geq 8$ such that for 
\begin{equation}\alpha \leq \frac{\omega (1-\gamma)}{C \tau}, \label{eqn:step-size} \end{equation}
the following is true: $d_t \leq B, \forall t \geq 0.$
\end{theorem}

\begin{proof}
We will prove this result via induction. For the base case of induction, note that we have already established in Lemma~\ref{lemma:base} that $d_k \leq B, \forall k \in [\tau]$. Now consider any $t \geq \tau$, and suppose that $d_k \leq B, \forall k \in [t].$ We will now show that under the requirement on the step-size $\alpha$ in the statement of the lemma, it holds that $d_{t+1} \leq B.$ To that end, starting from the main recursion in Eq.~\eqref{eqn:main_recursion}, we have (using $\sigma^2 \leq B$):
\begin{equation}
\begin{aligned}
    d_{t+1} &\leq {\left(1-2 \alpha \omega (1-\gamma) + 8 \alpha^2 \right) d_t} + {32 \alpha^2 B}+{2\alpha e_t} \\
    d_{t+1} & \overset{(a)}\leq \left(1-2 \alpha \omega (1-\gamma) + 8 \alpha^2 \right) d_t + O(\alpha^2 \tau B) \\
    & \overset{(b)}\leq \left(1-2 \alpha \omega (1-\gamma) + \alpha^2 (8+O(\tau))\right) B.
\end{aligned}
\nonumber
\end{equation}
For (a), we used the induction hypothesis in tandem with Lemma~\ref{lemma:mixing}; for (b), we invoked the induction hypothesis again. We conclude that there exists some universal constant $C \geq 8$ such that
$$ d_{t+1} \leq \left(1-2 \alpha \omega (1-\gamma) + C \alpha^2 \tau \right) B \leq (1-\alpha \omega (1-\gamma)) B \leq B, $$
where we used the choice of the step-size in Eq.~\eqref{eqn:step-size}. This establishes the induction claim and completes the proof.  
\end{proof}

The final convergence rate for TD learning now follows almost immediately. 

\begin{theorem}
\label{thm:mainbnd}
Suppose the step-size $\alpha$ is chosen as in Eq.~\eqref{eqn:step-size}. Then, the following is true for all $t\geq \tau$:
$$ d_{t+1} \leq (1-\alpha \omega (1-\gamma))d_t + O(\alpha^2 \tau B).$$
\end{theorem}
\begin{proof}
Now that we have argued in Theorem~\ref{thm:ind_main} that $d_t \leq B, \forall t \geq 0$, we can appeal to Lemma~\ref{lemma:mixing} to conclude that $e_t \leq O(\alpha \tau B), \forall t \geq \tau.$ Plugging this bound on $e_t$ back in the main recursion Eq.~\eqref{eqn:main_recursion} leads to the desired claim. 
\end{proof}

Theorem~\ref{thm:mainbnd} tells us that with a constant step-size, the iterates generated by \texttt{TD}(0) converge exponentially fast (in the mean-square sense) to a ball of radius $O(\alpha B)$ around the optimal parameter $\theta^*$. While we claim no novelty for this result, what is novel is how we arrive at it using our inductive technique. 

A couple of remarks are in order regarding the choice of the step-size in Theorem~\ref{thm:ind_main}. 

\begin{remark}
The fact that a valid choice of $\alpha$ satisfying Eq.~\eqref{eqn:step-size} does always exist follows from noting that $\tau_{\alpha} = K \log(1/\alpha)$ for some constant $K>0$, and that $\alpha \log(1/\alpha)$ can be made arbitrarily small by making $\alpha$  suitably small. 
\end{remark}

\begin{remark}
\label{rem:Remark1}
Note that the choice of the step-size in Eq.~\eqref{eqn:step-size} requires knowledge of the mixing-time $\tau$. Since the underlying MDP is unknown, the availability of such knowledge might appear restrictive. That said, requiring the step-size to scale inversely with the mixing time is not exclusive to our work; instead, such a requirement shows up in all other finite-time analysis papers on TD learning and stochastic approximation (under Markovian sampling) we are aware of that do not assume a projection step~\cite{srikant2019finite, chenQ}. Interestingly, however, by assuming a projection step, the need for such a requirement is bypassed in~\cite{bhandari_finite}.

\end{remark}

By using a carefully weighted combination of the iterates, one can obtain a finer convergence result relative to Theorem~\ref{thm:mainbnd}; to state this result, we use the notation $\Vert x \Vert_D = \sqrt {x^{\top} D x}$, where recall that $D$ is the diagonal matrix containing the entries of the stationary distribution $\pi$ along the diagonal. The specific form of the next result appears to be new for TD learning.

\begin{theorem}
\label{thm:avgiterate} {Define $A \triangleq 0.5 \omega (1-\gamma)$, $\bar{w}_t \triangleq (1-\alpha A)^{-(t+1)}, \forall t \geq 0$, and set  $w_t = \bar{w}_t/W_T$, where $W_T=\sum_{t=0}^{T} \bar{w}_t$.} There exists a constant step-size $\alpha$ satisfying the condition in Eq.~\eqref{eqn:step-size}, such that the following is true:
\begin{equation}
\begin{aligned}
\mathbb{E}[{\Vert  \hat{V}_{\bar{\theta}_T} - \hat{V}_{\theta^*} \Vert}^2_D] &\leq C_1 \exp{ \left(-\frac{\omega^2 (1-\gamma)^2 (T+1)}{2C\tau}\right)}\\
&\hspace{3mm}+ \tilde{O}\left(\frac{ \tau B}{\omega^2 (1-\gamma)^2 (T+1)} \right),
\end{aligned}
\label{eqn:avgiteratebnd}
\end{equation}
where $C\geq 8$ is the same universal constant as in Theorem~\ref{thm:ind_main}, $B=10 \max\{\Vert \theta_0 - \theta^* \Vert^2, \sigma^2\}$, $\bar{\theta}_T=\sum_{t=0}^{T} w_t \theta_t,$ and
$$ C_1 = {O}\left( \frac{B\tau}{(\omega^3 (1-\gamma)^3)}\right).$$

\end{theorem}

The proof of the above result is a simple adaptation of that of Lemma 25 in~\cite{stich2020comm}. We spell out the parts unique to our setting in the Appendix. {A couple of points are worth mentioning. First, we note that to achieve the bound in~\eqref{eqn:avgiteratebnd}, the specific form of averaging we employ is different from the more commonly studied \emph{Polyak-Ruppert} averaging in SA~\cite{borkar2021ode, huo2023bias, lauand2023}. Second, recent work~\cite{huo2023bias, lauand2023} has shown that with a constant step-size $\alpha$, the mean-square error of TD learning will exhibit an asymptotic bias on the order of $O(\alpha)$ that cannot, in general, be removed by iterate-averaging. Our bound in Eq.~\eqref{eqn:finalbndavg} aligns with this observation. As the proof of Theorem~\ref{thm:avgiterate} reveals, $\alpha$ needs to scale inversely with $T$ for large $T$ to arrive at~\eqref{eqn:avgiteratebnd}.}

Before proceeding further, let us quickly distill the main steps of our argument.

\begin{itemize}
    \item \textbf{Step 1.} We used the contraction property and the Lipschitz property of the \texttt{TD}(0) update direction to set up the main recursion in Eq.~\eqref{eqn:main_recursion}.

    \item \textbf{Step 2.} Viewing the effect of Markovian noise as a disturbance/perturbation, our goal was to then show that this perturbation is uniformly bounded. To that end, we developed our novel inductive argument and showed that the iterates generated by \texttt{TD}(0) remain uniformly bounded in expectation. 

    \item \textbf{Step 3.} We went back to the main recursion from Step 1, but this time having proven that the disturbance is uniformly bounded. The rest is straightforward. 
\end{itemize}

As we will discuss in the next section, the above steps constitute a general recipe for analyzing (potentially nonlinear) contractive stochastic approximation algorithms. Before we do so, it is instructive to elaborate on how our analysis relates to existing work on this topic. 

\textbf{Comments on our analysis.} Step 1 of our analysis builds on Lemma~\ref{lemma:convex} from~\cite{bhandari_finite}, which, in turn, is inspired from prior results in~\cite{tsitsiklisroy}. We provide a proof of Lemma~\ref{lemma:convex} - different from that in~\cite{bhandari_finite} - in the Appendix. The idea of conditioning sufficiently into the past to exploit the geometric mixing property of the underlying Markov chain is quite standard by now~\cite{bhandari_finite, srikant2019finite, chenQ}. We use this idea in Lemma~\ref{lemma:mixing}. The main distinguishing feature of our analysis relative to prior work in~\cite{bhandari_finite} and~\cite{srikant2019finite} is how we handle the disturbance term $T_3$ in the main recursion Eq.~\eqref{eqn:main_recursion}. The projected TD algorithm studied in~\cite{bhandari_finite} automatically ensures that the iterates remain uniformly bounded. This considerably simplifies the process of controlling $T_3$. A close inspection of the proof in~\cite{srikant2019finite} reveals that Lemma 3 in their paper plays a crucial role in analyzing the \emph{unprojected} version of \texttt{TD}(0). This lemma provides a bound of the following form: 
\begin{equation}\label{eqn:srikantlemma}
   \Vert \theta_t - {\theta}_{t-\tau} \Vert \leq O(\alpha\tau)( \Vert \theta_t \Vert + \sigma), \forall t\geq \tau. 
\end{equation}

In words, this lemma relates the change in the iterates over the interval $[t-\tau, t]$ to the current iterate. Versions of Lemma 3 from \cite{srikant2019finite} have also appeared in follow-up works to study Q-learning with linear function approximation~\cite{chenQ}. At a technical level, one of our main contributions is to show that one can analyze unprojected \texttt{TD}(0) \emph{without this lemma.} We will revisit the significance of this point again in Section~\ref{sec:applications}. To the best of our knowledge, the idea of first arguing uniform boundedness of the iterates via induction, and then using this fact to derive uniform bounds on the disturbance term $T_3$ (due to Markov noise) has not appeared before. We also note that the main claim of boundedness in Theorem~\ref{thm:ind_main} does not require a decaying step-size; rather, the choice of constant step-size in Theorem~\ref{thm:ind_main} complies with the standard choice of step-size in~\cite{bhandari_finite} and~\cite{srikant2019finite}. Finally, we would like to draw attention to the works \cite{gosavi, beck, qu} that provide inductive arguments to establish boundedness of the iterates for constant step-size Q-learning. However, the analyses in these papers only apply to a \emph{tabular} setting without function approximation. Whether similar inductive arguments could be developed for contractive stochastic approximation algorithms with function approximation was unclear prior to our work. 

\section{Applications of our Analysis Technique}
\label{sec:applications}
In this section, we briefly discuss a few applications to demonstrate that the scope of the inductive proof technique outlined in Section~\ref{sec:analysis} extends well beyond the \texttt{TD}(0) algorithm with linear function approximation. 

$\bullet$ \textbf{Nonlinear Stochastic Approximation.} In a typical non-linear stochastic approximation (SA) problem, the goal is to solve for a parameter $\theta^*$ such that $\bar{g}(\theta^*) = 0$, where $\bar{g}(\theta)= \mathbb{E}_{X \sim \pi}[g(\theta; X)]$; one can interpret this as a root-finding problem.  Here, $X$ is a noise random variable that comes from a statistical sample space $\mathcal{X}$, and has distribution $\pi$. Importantly, $\pi$ is assumed to be \emph{unknown} (hence, the learning aspect). The function $g: \mathcal{X} \times \mathbb{R}^d \mapsto \mathbb{R}^d$ is a general nonlinear mapping. The learner has access to $\bar{g}(\cdot)$ only through the noisy samples $\{g(\cdot \,; X_t)\}$, where $\{X_t\}$ is generated from a finite-state Markov chain that is aperiodic and irreducible with stationary distribution $\pi$. The celebrated SA protocol for finding $\theta^*$ takes the form: 
\begin{equation}
\theta_{t+1} = \theta_{t}+\alpha_t g(\theta_t; X_t),
\label{eqn:SA}
\end{equation}
where $\{\alpha_t\}$ is the learning-rate (step-size) sequence. Algorithms within the \texttt{TD}($\lambda$) family are instances of linear SA - a special case of the formulation above - where $ g(\theta; X_t)$ is affine in the parameter $\theta$. To extend our results beyond the linear SA setting, we make the following standard assumptions. 

\begin{assumption} \label{ass:lips} There exist $L, \sigma \geq 1$ s.t. $\Vert g(\theta_1; X)-g(\theta_2;X) \Vert \leq L \Vert \theta_1 - \theta_2 \Vert, \forall \theta_1,\theta_2 \in \mathbb{R}^{d}$, $ \forall X \in \mathcal{X}$, and 
\begin{equation}
\Vert g(\theta; X) \Vert \leq L \left(\Vert \theta \Vert + \sigma \right), \forall \theta \in \mathbb{R}^d, \forall X \in \mathcal{X}.  
\label{eqn:gennormgrad}
\end{equation}
\end{assumption}

\begin{assumption} \label{ass:diss}  The equation $\bar{g}(\theta)=0$ has a solution $\theta^*$, and $ \exists \beta >0$ s.t. 
\begin{equation}
    \langle \theta - \theta^*, \bar{g}(\theta) - \bar{g}(\theta^*) \rangle \leq - \beta \Vert \theta-\theta^* \Vert^2, \forall \theta \in \mathbb{R}^d. 
\end{equation}
\end{assumption}

Assumption~\ref{ass:lips} tells us that $g(\theta; X)$ is globally uniformly (w.r.t. $X$) Lipschitz in the parameter $\theta$. This assumption is met by TD- and Q-learning with linear function approximation~\cite{bhandari_finite,srikant2019finite, chenQ}, and is typical in the analysis of stochastic optimization~\cite{doanSGD}. Assumption~\ref{ass:diss} is referred to as the strong monotone property of the operator $\bar{g}(\theta)$, and is directly responsible for exponentially fast convergence (to $\theta^*$) of the steady-state version of~\eqref{eqn:SA}. We note that this strong monotone property is satisfied by TD-learning with linear function approximation~\cite{bhandari_finite,srikant2019finite}, variants of Q-learning with linear function approximation~\cite{chenQ}, and strongly convex loss functions in the context of optimization. We have the following result. 

\begin{theorem}
\label{thm:nonlinearSA}
    Let $\bar{\beta} \triangleq \min\{\beta, 1/\beta\},$ and $B \triangleq 10 \max\{\Vert \theta_0 - \theta^* \Vert^2, \sigma^2\}$. Suppose Assumptions~\ref{ass:lips} and~\ref{ass:diss} hold. Then, there exists a universal constant $C \geq 8$ such that for $\alpha \leq \frac{\bar{\beta}}{C \tau L^2}$, the iterates generated by Eq.~\eqref{eqn:SA} satisfy $d_t \leq B, \forall t \geq 0$, and 
    $$ d_{t+1} \leq (1-\alpha \beta)d_t + O(\alpha^2 L^2 \tau B),$$
where $d_t \triangleq \mathbb{E}[ \Vert \theta_t - \theta^*\Vert^2 ]$, and $\tau=\tau_{\alpha}$ is the mixing time as defined in Definition~\ref{def:mix}. 
\end{theorem}

The proof of this result is identical to that of Theorems~\ref{thm:ind_main} and~\ref{thm:mainbnd}, and the only additional work pertains to keeping track of how the Lipschitz parameter $L$ propagates through the bounds. We omit repeating the same arguments here again. 

\textbf{Main Takeaway.} The main message conveyed by Theorem~\ref{thm:nonlinearSA} is that the simple analysis recipe we outlined in Section~\ref{sec:analysis} - involving our novel inductive argument - carries over seamlessly to a broad class of stochastic approximation algorithms that cover the linear \texttt{TD}($\lambda$) family, variants of Q-learning, and smooth, strongly convex stochastic optimization under Markovian noise. In particular, Lipschitzness and some notion of contractivity of the underlying operator appear to be enough for our induction argument to go through. 
\vspace{2mm}

$\bullet$ \textbf{Stochastic Approximation with Perturbations.} Motivated by the question of robustness of iterative RL algorithms to structured perturbations, let us consider the following \emph{inexact} SA scheme shown below:
\begin{equation}
\theta_{t+1} = \theta_{t}+\alpha \tilde{g}_t,
\label{eqn:inexactSA}
\end{equation}
where $\tilde{g}_t$ is a perturbed version of $g(\theta_t; X_t)$. To convey our key points, let us consider a specific type of perturbation introduced by delays, where $\tilde{g}_t = g(\theta_{t-\tau_t}; X_{t-\tau_t})$, and $0 \leq \tau_t \leq t$ is a time-varying (potentially random) delay that is uniformly bounded, i.e., $\tau_t \leq \tau_{\max}, \forall t \geq 0$, where $\tau_{\max}$ is some positive integer. Such delays are usually unavoidable in the context of distributed/networked learning problems where information (e.g., models and model-differentials) gets exchanged over imperfect channels, and, as such, have been extensively studied in the context of optimization with i.i.d. data. However, there is little to no work providing an understanding of how arbitrary time-varying (albeit bounded) delays affect the finite-time performance of SA schemes driven by Markovian noise. Now suppose we try to invoke some variant of Lemma 3 from~\cite{srikant2019finite} to account for Markovian sampling. In this context, a bound of the form in Eq.~\eqref{eqn:srikantlemma} is no longer applicable, since due to the presence of delays, $\Vert \theta_t - \theta_{t-\tau} \Vert$ \emph{is not just a function of the current iterate, but several other iterates from the past.} 

The above discussion tells us why existing proof techniques for SA under Markov noise do not immediately lend themselves to the analysis of perturbed SA schemes, where the perturbation can contain terms from the past. Now suppose we rewrite Eq.~\eqref{eqn:inexactSA} in the following way:
\begin{equation}
\theta_{t+1} = \underbrace{\theta_t + \alpha \bar{g}(\theta_t)}_{\mathcal{A}}+ \underbrace{\alpha \left(g(\theta_t; X_t)-\bar{g}(\theta_t)\right)}_{e_{1,t}} + \underbrace{\alpha \left(\tilde{g}_t - g(\theta_t; X_t)\right)}_{e_{2,t}}. 
\end{equation}

In the above decomposition, $e_{1,t}$ and $e_{2,t}$ can be each viewed as a disturbance to the nominal steady-state dynamics of Eq.~\eqref{eqn:SA}, as captured by the term $\mathcal{A}.$ Here, while $e_{1,t}$ arises from Markovian sampling, $e_{2,t}$ captures the effect of delays (or more generally, some other perturbation). The main message is that we can lump these iterate-dependent disturbances together, and use an inductive argument - akin to what we did in Section~\ref{sec:analysis} - to establish uniform bounds on them in expectation. Once this is done, we are again back to a scenario where a uniformly bounded additive perturbation hits the steady-state nominal dynamics. While we flesh out these details for delayed SA in a companion paper~\cite{delayedSA}, \emph{the scope of the above argument is by no means just limited to perturbations arising from delays}. In principle, as long as we can use induction to argue that the perturbation is on the order of $O(\alpha^2)$, the approach in Section~\ref{sec:analysis} will go through. 

\section{Conclusion}
We provided a simple and self-contained finite-time analysis of TD learning with linear function approximation based on a novel inductive argument. We showed that our proof technique extends to more general nonlinear SA schemes, and can be used to analyze inexact SA schemes with perturbations. The relative simplicity of our overall approach opens up various interesting possibilities. We discuss some of them below.

\begin{enumerate} 

\item In a recent work~\cite{tianNN}, the authors investigate the finite-time performance of TD learning with neural-network based function approximators. Their analysis builds on their prior work~\cite{liuTD} where an interesting ``gradient-splitting" interpretation is provided for the TD update direction. However, both~\cite{liuTD} and~\cite{tianNN} require a projection step in the algorithm to ensure stability of the iterates. Whether our inductive technique to guarantee uniform boundedness of the iterates (in expectation) carries over to neural function approximators remains to be seen.

\item Our results in this paper focus exclusively on single-time-scale SA algorithms. It would be interesting to see if similar simpler proofs can be developed for two-time-scale SA schemes in the context of RL. 

\item As we described in the main paper, one immediate benefit of our technique is that it can allow for handling multiple forms of disturbances/perturbations in the update rule simultaneously. In this context, we briefly talked about handling perturbations in the form of delays. We plan to explore how one can use our technique to study other types of perturbations typical in large-scale problems, such as those in~\cite{mitraTDEF} involving aggressive compression. 

\item As alluded to in Remark~\ref{rem:Remark1}, the design of the step-size in our work, and other relevant papers on TD learning, requires some knowledge of the underlying Markov chain induced by the policy to be evaluated, whether it is in the form of the mixing time $\tau$ and/or the smallest eigenvalue $\omega$ of the matrix $\Sigma$. Instead of assuming such knowledge ahead of time, one can potentially estimate $\tau$ and $\omega$ from samples, using, for instance, the ideas in~\cite{mixing1} and~\cite{mixing2}. These estimates can then be used to design an appropriate step-size sequence. We conjecture that the analysis of such a scheme will be quite non-trivial. To see why, note that if the step-size at each time-step is designed based on data, then it will become a random object that inherits randomness from the underlying Markov chain. This will in turn create further complex correlations between the iterate, the step-size, and the data tuples. We are unaware of any existing work that studies such stochastic dependencies.  

\item Finally, extensions of our technique to the multi-agent setting~\cite{wang2023federated, khodadadian}  would also be interesting to pursue. 
\end{enumerate}

\appendix
\section{Omitted Proofs}
\noindent \textbf{Proof of Lemma~\ref{lemma:convex}}. Our proof will leverage the following result from~\cite{tsitsiklisroy}.

\begin{lemma}
\label{lemma:roytsit}
Given any $x \in \mathbb{R}^n$, the following is true:
$$ \Vert P_{\mu} x \Vert_D \leq \Vert x \Vert_D. $$
\end{lemma}

We will also make use of the fact that $\bar{g}(\theta)=\bar{A}\theta-\bar{b}$, where  $\bar{A} = \Phi^{\top} D \left(\gamma P_{\mu} - I \right) \Phi$, and $\bar{b} = - \Phi^{\top} D R_{\mu}$~\cite{tsitsiklisroy}. Using $\bar{g}(\theta^*)=0$ and $\hat{V}_{\theta} = \Phi \theta$, we then have  
\begin{equation}
\begin{aligned}
\langle \theta - \theta^*, \bar{g}(\theta) \rangle &= \langle \theta - \theta^*, \bar{g}(\theta) - \bar{g}(\theta^*) \rangle\\
&= \langle \theta - \theta^*, \bar{A} \left(\theta - \theta^*\right) \rangle\\
&=\gamma (\theta - \theta^*)^{\top} \Phi^{\top} D P_{\mu} \Phi (\theta - \theta^*)\\
&\hspace{2.5mm} - (\theta - \theta^*)^{\top} \Phi^{\top} D \Phi (\theta - \theta^*)\\
&=\underbrace{\gamma (\hat{V}_{\theta} - \hat{V}_{\theta^*})^{\top} D P_{\mu} (\hat{V}_{\theta} - \hat{V}_{\theta^*})}_{(*)} \\
&\hspace{2.5mm} - \underbrace{(\hat{V}_{\theta} - \hat{V}_{\theta^*})^{\top} D (\hat{V}_{\theta} - \hat{V}_{\theta^*})}_{(**)}.\\
\end{aligned}
\label{eqn:Lemmaconvex}
\end{equation}
Observe that $(**) = \Vert \hat{V}_{\theta} - \hat{V}_{\theta^*} \Vert^2_D.$ {Next, note that given any two vectors $x,y \in \mathbb{R}^n$, it holds that $x^{\top} D y = x^{\top} D^{1/2} D^{1/2} y \leq \Vert D^{1/2} x \Vert \Vert D^{1/2} y \Vert = \Vert x \Vert_D \Vert y \Vert_D.$ Applying this to $(*)$ yields: 
$$(*) \leq \gamma {\Vert  (\hat{V}_{\theta} - \hat{V}_{\theta^*}) \Vert}_D {\Vert P_{\mu} (\hat{V}_{\theta} - \hat{V}_{\theta^*}) \Vert}_D 
\leq \gamma \Vert \hat{V}_{\theta} - \hat{V}_{\theta^*} \Vert^2_D, $$ where the second inequality follows from Lemma~\ref{lemma:roytsit}.} Combining this fact with Eq.~\eqref{eqn:Lemmaconvex}, we obtain
\begin{equation}
\begin{aligned}
\langle \theta - \theta^*, \bar{g}(\theta) \rangle &\leq - (1-\gamma) \Vert \hat{V}_{\theta} - \hat{V}_{\theta^*} \Vert^2_D\\
&= -(1-\gamma) (\theta- \theta^*)^{\top} \Phi^{\top} D \Phi (\theta- \theta^*)\\
&= - (1-\gamma) (\theta- \theta^*)^{\top} \Sigma (\theta- \theta^*)\\
& \leq - \omega (1-\gamma) \Vert \theta - \theta^*\Vert^2,
\end{aligned}
\end{equation}
where in the last step, we used that $\Sigma$ is positive definite with smallest eigenvalue $\omega.$ This completes the proof of Lemma~\ref{lemma:convex}.

\noindent \textbf{Proof of Theorem~\ref{thm:avgiterate}}. 
Our first goal is to get an estimate of $e_t$ in Eq.~\eqref{eqn:disturbance} for $t \in [\tau-1].$ To that end, fix any $t\in [\tau -1]$, and observe that
\begin{equation}
\begin{aligned}
\langle \theta_t -\theta^*, g_t(\theta_t) - \bar{g}(\theta_t)\rangle &\leq \Vert \theta_t - \theta^*\Vert \left(\Vert g_t(\theta_t) \Vert + \Vert \bar{g}(\theta_t)\Vert \right)\\
& \leq 4 \Vert \theta_t - \theta^* \Vert \left(\Vert \theta_t - \theta^* \Vert + \sigma\right)\\
& \leq 8B,
\nonumber
\end{aligned}
\end{equation}
where in the second inequality, we used equations~\eqref{eqn:nrm} and~\eqref{eqn:nrmbnd}, and for the last inequality, we used $\Vert \theta_t - \theta^* \Vert \leq \sqrt{B}, \forall t \in [\tau-1]$ from Lemma~\ref{lemma:base}. Thus, $e_t \leq 8B, \forall t \in [\tau-1].$ Plugging this bound in Eq.~\eqref{eqn:main_recursion}, we obtain
\begin{equation}
d_{t+1} \leq (1-\alpha \omega (1-\gamma))d_t + O(\alpha B), \forall t \in [\tau-1],
\label{eqn:initbnd}
\end{equation}
where we used the choice of the step-size in Eq.~\eqref{eqn:step-size}. To proceed, we note that for any $\theta \in \mathbb{R}^K$, $\Vert \hat{V}_{\theta} - \hat{V}_{\theta^*} \Vert^2_D =(\theta-\theta^*)^{\top} \Sigma (\theta-\theta^*) \leq \Vert \theta -\theta^* \Vert^2$, since $\Sigma$ is positive definite with largest eigenvalue less than 1. Defining $s_t \triangleq \mathbb{E}[ \Vert \hat{V}_{\theta_t} - \hat{V}_{\theta^*} \Vert^2_D],$ recalling that $A = 0.5 \omega (1-\gamma)$, and using Eq.~\eqref{eqn:initbnd} in tandem with Theorem~\ref{thm:mainbnd}, we then have 
\begin{equation}
d_{t+1} \leq  
\begin{cases}
    (1-\alpha A)d_t -\alpha A s_t + c\alpha B, & t \in [\tau-1]\\
    (1-\alpha A)d_t -\alpha A s_t + c\alpha^2 \tau B, & t\geq \tau,
\end{cases}
\label{eqn:stage_bnd}
\end{equation}
where $c$ is some universal constant. Let us define a sequence of weights as $\bar{w}_t \triangleq (1-\alpha A)^{-(t+1)}, \forall t \geq 0$, and set $W_T=\sum_{t=0}^{T} \bar{w}_t.$ Using~\eqref{eqn:stage_bnd}, we then have
\begin{equation}
\begin{aligned}
A \left( \sum_{t=0}^T \frac{\bar{w}_t s_t}{W_T} \right) &\leq \underbrace{\frac{1}{\alpha W_T} \sum_{t=0}^{T} \left(\bar{w}_t (1-\alpha A) d_t - \bar{w}_t d_{t+1} \right)}_{(*)} \\
& \hspace{3mm} \underbrace{\frac{c B}{W_T} \sum_{t=0}^{\tau-1} \bar{w}_t}_{(**)} + \underbrace{\frac{c \alpha \tau B}{W_T} \sum_{t=\tau}^{T} \bar{w}_t}_{(***)}.  
\end{aligned}
\nonumber
\end{equation}
We now bound each term above. For $(*)$, we use $w_t (1-\alpha A) = w_{t-1}$ to obtain a telescoping sum, yielding
$$ (*) \leq \frac{(1-\alpha A) \bar{w}_0 d_0}{\alpha W_T} \leq \frac{d_0}{\alpha} \left(1-\alpha A\right)^{T+1},$$
where in the last step, we used $W_T \geq \bar{w}_T \geq (1-\alpha A)^{-(T+1)}.$ Next, it is easy to see that $(***) \leq c \alpha \tau B.$ Finally, we have
$$ (**) \leq \frac{c B}{\alpha A W_T} (1-\alpha A)^{-\tau} \leq \frac{2 c B}{\alpha A} (1-\alpha A)^{T+1}.$$
Here, we used that for $\alpha$ satisfying Eq.~\eqref{eqn:step-size}, it holds that $\alpha A \tau \leq 1/2.$ Thus, from Bernoulli's inequality, we have $(1-\alpha A)^{\tau} \geq (1- \alpha A \tau) \geq 1/2.$ Combining the bounds above, using $d_0 \leq B$, $A \leq 1$, and setting $w_t = \bar{w}_t/W_T,$ we have
\begin{equation}
\left(\sum_{t=0}^T w_t s_t \right) \leq \frac{3c B}{\alpha A^2} \exp\left(-\alpha A (T+1) \right) + \frac{c \alpha \tau B}{A}.
\label{eqn:finalbndavg}
\end{equation}
The rest of the proof involves tuning $\alpha$ carefully as in the proof of Lemma 25 in~\cite{stich2020comm}. We provide details for completeness. Let us define 
$$ 
\lambda \triangleq \max\{\exp(1), A (T+1)^2/\tau\}. 
$$
Now we consider two cases. \textbf{Case 1:} If 
$$\frac{\ln(\lambda)}{A (T+1)} \leq \frac{\omega (1-\gamma)}{C \tau}, \hspace{1mm} \textrm{set} \hspace{1mm} \alpha = \frac{\ln(\lambda)}{A (T+1)}.$$ 
Here, $C$ is as in Eq.~\eqref{eqn:step-size}. 
\textbf{Case 2:} If
$$  \frac{\omega (1-\gamma)}{C \tau} < \frac{\ln(\lambda)}{A (T+1)},  \hspace{1mm} \textrm{set} \hspace{1mm} \alpha = \frac{\omega (1-\gamma)}{C \tau}. $$

Observe that by choosing $\alpha$ in the manner above, one meets the requirement on the step-size in Eq.~\eqref{eqn:step-size} of Theorem~\ref{thm:ind_main} to ensure boundedness of the iterates. Now let us study how the choice of $\alpha$ above affects the bound in Eq.~\eqref{eqn:finalbndavg}. Consider Case 1 first. Direct substitution of $\alpha$ into the bound, and simplification using $\ln(\lambda) \geq 1$ yields:
$$ \left(\sum_{t=0}^T w_t s_t \right) \leq O\left( \frac{\tau B}{ \omega^2 (1-\gamma)^2 (T+1)}\right) + {O}\left( \frac{\tau B \ln(\lambda) }{ \omega^2 (1-\gamma)^2 (T+1)}\right) = \tilde{O}\left( \frac{\tau B}{ \omega^2 (1-\gamma)^2 (T+1)}\right),$$
where we used the definition of $A$. Proceeding similarly, for Case 2, we have:
$$ \left(\sum_{t=0}^T w_t s_t \right) \leq C_1 \exp{ \left(-\frac{\omega^2 (1-\gamma)^2 (T+1)}{2C\tau}\right)} + \tilde{O}\left(\frac{ \tau B}{\omega^2 (1-\gamma)^2 (T+1)} \right),$$
where 
$$ C_1 = {O}\left( \frac{B\tau}{\omega^3 (1-\gamma)^3}\right).$$
To arrive at the above bound, we used the fact that $\alpha < \frac{\ln(\lambda)}{A (T+1)}$ to bound $T_2$. Combining the bounds from Cases 1 and 2, and applying Jensen's inequality, we arrive at Eq.~\eqref{eqn:avgiteratebnd}. 
\bibliographystyle{unsrt}
\bibliography{refs}
\end{document}